\pdfoutput=1
\documentclass[letterpaper]{article} 
\usepackage{aaai2026}  
\usepackage{times}  
\usepackage{helvet}  
\usepackage{courier}  
\usepackage[hyphens]{url}  
\usepackage{graphicx} 
\usepackage{natbib}  
\usepackage{caption} 
\usepackage{algorithm}
\usepackage{algorithmic}
\usepackage{amsfonts}
 \usepackage{amsmath} 
\usepackage{newfloat}
\usepackage{listings}
\DeclareCaptionStyle{ruled}{labelfont=normalfont,labelsep=colon,strut=off} 
\floatstyle{ruled}
\newfloat{listing}{tb}{lst}{}
\floatname{listing}{Listing}
\title{Multi-Band Variable-Lag Granger Causality: A Unified Framework for Causal Time Series Inference across Frequencies}
\author{
    Chakattrai Sookkongwaree\textsuperscript{\rm 1},
    Tattep Lakmuang\textsuperscript{\rm 2},
    Chainarong Amornbunchornvej\textsuperscript{\rm 3}
}
\affiliations{
    \textsuperscript{\rm 1}Chulalongkorn university\\
    \textsuperscript{\rm 2}Korea Advanced Institute of Science \& Technology\\
    \textsuperscript{\rm 3}National Electronics and Computer Technology Center\\

    6632033821@student.chula.ac.th, ken\_tattep@kaist.ac.kr, chainarong.amo@nectec.or.th 
}

\usepackage{bibentry}

\begin{document}

\maketitle

\begin{abstract}
Understanding causal relationships in time series is fundamental to many domains, including neuroscience, economics, and behavioral science. Granger causality is one of the well-known techniques for inferring causality in time series. Typically, Granger causality frameworks have a strong fix-lag assumption between cause and effect, which is often unrealistic in complex systems. While recent work on variable-lag Granger causality (VLGC) addresses this limitation by allowing a cause to influence an effect with different time lags at each time point, it fails to account for the fact that causal interactions may vary not only in time delay but also across frequency bands. For example, in brain signals, alpha-band activity may influence another region with a shorter delay than slower delta-band oscillations. In this work, we formalize Multi-Band Variable-Lag Granger Causality (MB-VLGC) and propose a novel framework that generalizes traditional VLGC by explicitly modeling frequency-dependent causal delays.  We provide a formal definition of MB-VLGC,  demonstrate its theoretical soundness, and propose an efficient inference pipeline.   Extensive experiments across multiple domains demonstrate that our framework significantly outperforms existing methods on both synthetic and real-world datasets, confirming its broad applicability to any type of time series data. Code and datasets are publicly available.
\end{abstract}

\section{Introduction}

\begin{figure*}[t]
\centering
\includegraphics[width=1.5\columnwidth]{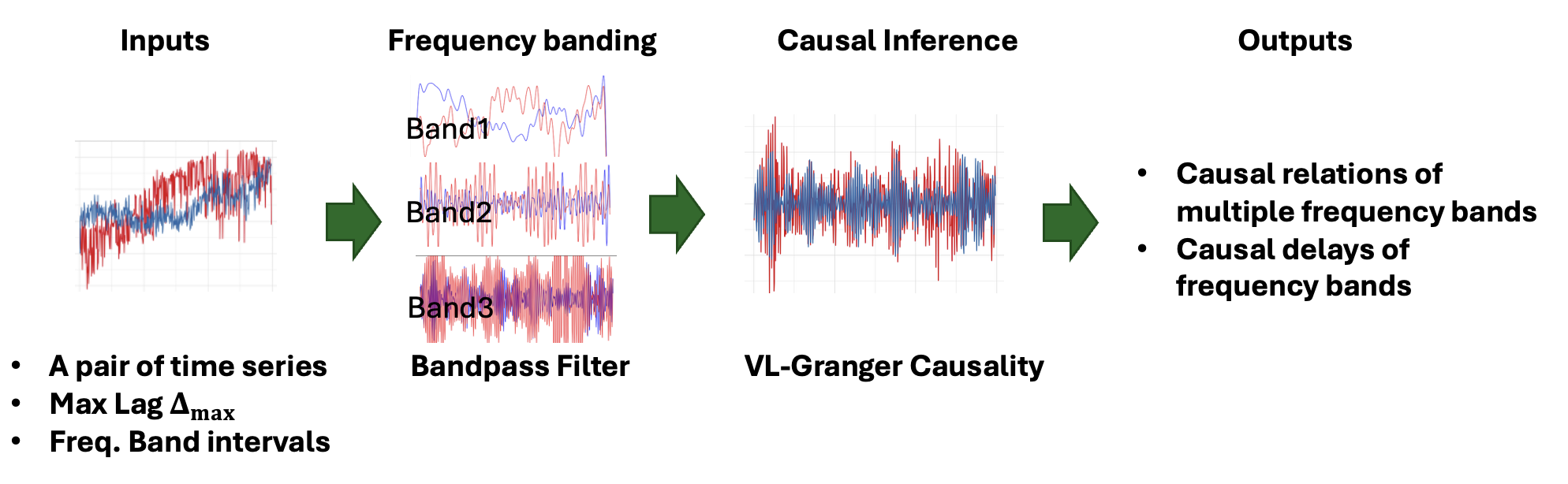} 
\caption{A high-level overview of proposed framework. Given a pair of time series, $\Delta_{max}$, and band intervals as inputs, there are three steps in the framework: Band Decomposition, VL-Granger test for each band, and result integration.}
\label{fig:overviewframework}
\end{figure*}



Understanding causal relationships in time series data is fundamental across numerous scientific disciplines, including neuroscience, economics, and behavioral science. One of important questions in these domains is identifying which time-dependent signals initiate or influence other patterns of behavior over time. Experimental methods such as randomized controlled trials, in many cases, are infeasible due to ethical, logistical, or financial constraints~\cite{varian2016causal}. Consequently, methods for inferring causality from observational time series data have become crucial roles for scientific inquiry. 

One of the well-known frameworks for such analysis is Granger causality~\cite{granger1969investigating}. It defines a directional causal influence from a time series $X$ to another time series $Y$ if the inclusion of $X$'s past improves the prediction of $Y$ beyond the predictive power of $Y$'s own past. Despite its popularity, traditional Granger causality methods rely on a critical assumption: the time delay between cause and effect is fixed. However, in many phenomena, the delay between a cause and its effect may vary over time, rendering the fixed-lag assumption overly restrictive.

To address this issue, the concept of Variable-Lag Granger Causality (VLGC)~\cite{10.1145/3441452} was introduced. VLGC allows for temporal flexibility by accommodating time-varying delays in causal relationships using Dynamic Time Warping (DTW)~\cite{1163055}. Nevertheless, VLGC overlooks a crucial aspect of frequency structures in signals. For example, in neuroscience, causal effects often vary across frequency bands~\cite{canolty2010functional}; for instance, alpha-band oscillations in the brain may exert influence more rapidly than slower delta-band activity~\cite{michalareas2016alpha}.

Hence, in this work, we formalize \textit{Multi-Band Variable-Lag Granger Causality (MB-VLGC)}, a unified framework that generalizes VLGC by modeling frequency-specific time delays. In addition to standard Granger causality and variable-lag methods, our framework offers:

\begin{itemize}
    \item \textbf{Unified time-frequency causal inference:} MB-VLGC extends both Variable-Lag Granger Causality (VLGC) and frequency-domain GC by combining spectral decomposition with dynamic temporal alignment.
    \item \textbf{Frequency-specific delay modeling:} Our framework infers causal interactions with distinct time lags across multiple frequency bands.
\end{itemize}

MB-VLGC reveals multiscale causal structures and offers a general tool for analyzing complex time series across scientific domains.

\subsection{Related works}
Granger causality (GC) is a widely adopted statistical framework for time series causality, establishing that a signal $X$ is causal to $Y$ if the past of $X$ improves the prediction of $Y$ beyond what $Y$’s own past can provide~\cite{granger1969investigating}. While GC is widely used in economics and science, GC is simple, which make it have many limitations. 

There were many methods were developed to overcome several issues of GC. For instance, since GC does not indicate how an influence may be distributed across frequencies, frequency-domain extensions of GC, such as Geweke’s spectral causality \citep{Geweke1982} and Hosoya’s decomposition \citep{hosoya1991decomposition}, were introduced and they enable analysis of frequency-specific causal interactions. For linearity aspect of GC, methods (e.g. Transfer Entropy (TE)~\cite{Schreiber2000,BEHRENDT2019100265}, PCMCI~\cite{runge2020discovering}) were developed to capture non-linear dependencies. Nevertheless, all methods mentioned above relied on fixed-lag conditionals;  they lack the capacity to explicitly model time-varying delays.

Regarding of time-varying delays, in many complex systems, the delay of influence can vary over time, violating this fixed-lag assumption. To relax this constraint, the work in~\cite{10.1145/3441452} proposed the concept of Variable-Lag Granger Causality (VLGC) and Variable-Lag Transfer Entropy (VLTE), which allow flexible, time-varying delays by aligning time series via DTW. VLGC and VLTE generalize classical Granger causality by permitting different lag lengths at different times, thereby capturing causal relationships that shift in time. Both approaches are also able to deal with non-stationary and unstable-causal-structure time series\footnote{Given $X$ causes $Y$,  the unstable-causal-structure occurs when $\exists t\neq t'$ and $Y_t=f(X), Y_{t'}=f'(X)$ s.t. $ f\neq f'$}. However, existing GC methods (including VLGC and frequency-domain approaches) do not account for frequency-specific variations in causal influence.

In other words, all previous works either infer variable lags or frequency-specific effects, but not both. Hence, we propose MB-VLGC  to fill this gap by unifying dynamic temporal alignment with spectral decomposition, enabling accurate causal inference across time and frequency.
\newtheorem{definition}{Definition}
\newtheorem{theorem}{Theorem}[section]
\newtheorem{lemma}[theorem]{Lemma}
\newtheorem{proposition}[theorem]{Proposition}
\newtheorem{corollary}[theorem]{Corollary}

\newenvironment{proof}[1][Proof]{\begin{trivlist}
\item[\hskip \labelsep {\bfseries #1}]}{\end{trivlist}}
\newenvironment{example}[1][Example]{\begin{trivlist}
\item[\hskip \labelsep {\bfseries #1}]}{\end{trivlist}}
\newenvironment{remark}[1][Remark]{\begin{trivlist}
\item[\hskip \labelsep {\bfseries #1}]}{\end{trivlist}}

\newcommand{\qed}{\nobreak \ifvmode \relax \else
      \ifdim\lastskip<1.5em \hskip-\lastskip
      \hskip1.5em plus0em minus0.5em \fi \nobreak
      \vrule height0.75em width0.5em depth0.25em\fi}

\section{Multi-Band Variable-Lag Granger Causality formalization}

\begin{definition}{Granger causality under stationary and VAR assumption}\\
 Let $U_t:=[X_t,Y_t]^T$ be a zero-mean, stationary, multivariate VAR process: $U_t = \sum_{k=1}^{p}{A_kU_{t-k}}+\varepsilon_t$, where $\varepsilon_t\sim iid(0,\Sigma) $ Let $G(z)$ be a stable, invertible multivariate linear filter with no cross-component coupling. Then for all $\omega \in [0,\pi]$, the spectral Granger causality from $Y\rightarrow X$ is invariant under filtering: $f_{Y\rightarrow X}(\omega)=f_{\tilde{Y}\rightarrow \tilde{X}}(\omega)$  
\end{definition}
As shown in Barnett (2011), under the assumptions of a stationary and a VAR, Granger causality is invariant under filtering. Given that filtering does not facilitate the separation of Granger causality by frequency, as described in Definition 1. In this work, we relax the assumptions of stationary and VAR, and investigate the properties of Granger causality under filtering.
\begin{proposition}
If the stationary and/or VAR assumption is dropped, then in general $f_{Y\rightarrow X}(\omega)\neq f_{\tilde{Y}\rightarrow \tilde{X}}(\omega)$: cannot conclude that the filtering does not help to separate Granger causality by frequency.
\end{proposition}


\begin{proof}
\indent Let $X_t := X_{t-1} + \varepsilon_t^x$, $Y_t := \varepsilon^y_t$  where $\varepsilon_t^x,\varepsilon^y_t \sim iid\mathcal{N}(0,1)$ This is a nonstationary process, not representable by a stationary VAR process with properties $E[X_t] = E[X_{t-1}] + E[\varepsilon_t^x]=E[X_{t-1}]$ which is constant and $Var(X_t) = Var(X_{t-1}) +1$ then $Var(X_t)=t$. 
\\Define the filter $G(L):=1-L$, where $L$ is a lag operator, which is linear, time invariant, diagonal, stable, and invertible in $\mathbb{Z}/\{1\}$ and define $\tilde{X_t}:= G(L)X_t=X_t-X_{t-1}=\varepsilon^x_t$ and $\tilde{Y_t} :=G(L)Y_t=\varepsilon_t^y-\varepsilon^y_{t-1}$ So: $\tilde{x_t}\sim iid\mathcal{N}(0,1)$ and $\tilde{y_t}\sim MA(1)$. We knew from the \cite{granger1974spurious} that $f_{Y \rightarrow X}(\omega)>0$ from a spurious Granger causality. And from $Cov(\tilde{X_t},\tilde{Y}_{t-k})=0, \forall k \geq 1$ makes $f_{\tilde{Y}\rightarrow \tilde{X}}(\omega)=0$. So, we cannot conclude that filtering does not help separate Granger-causality by frequency.
\end{proof}
According to Proposition 2.1, when the assumptions of stationary and a VAR representation are relaxed, filtering may facilitate the separation of Granger causality by frequency. To investigate this, we use the Variable-Lag Granger Causality (VLGC) method described in Definition 2, using a base case in which the signals are decomposed into separate frequency bands: $X = X_1 + X_2$ and $Y = Y_1 + Y_2$. To validate the reliability of the VL-Granger causality test on these separated bands, we analyze the associated residual variances, as discussed in Proposition 2.2.

\begin{definition}{VL-Granger Causality}
\label{def:VLGC}
\indent \\ For time series $X$ and $Y$, and an upper bound of time lag on any two time series to be tested or the maximum time lag $ \delta_{max} \in \mathbb{N}$. The residual of the regression $r^*_{YX}$ can be defined as follow: $$r^*_{YX}(t)=Y(t)-\sum^{\delta_{max}}_{i=1}(a_iY(t-i)+b_iX(t-i)+c_iX^*(t-i))$$ Where $X^*(t-i) = X(t-i+1-\Delta_{t-i+1})$ with a time delay parameter determined by the optimal alignment path $P^*$ that minimize the regression residual between $X$ and$Y$, $\Delta_t > 0$. The coefficients $a_i,b_i$ and $c_i$ are estimated such that the residuals $r_Y,r_{YX},$ and $r^*_{YX}$ are minimized. We can conclude that $X$ VL-Granger causes $Y$ if and only if $Var(r^*_{YX})$ is less than both $Var(r_{Y})$ and $Var(r_{YX})$.
\end{definition}
Next, we extend the concept of VLGC in Def.~\ref{def:VLGC} to work on multiple frequency bands.

\begin{definition}{Multi-Band Variable-Lag Granger Causality (MB-VLGC)}
    Given $X,Y$ and a set of interval bands $\mathcal{B}=\{B_1,\dots,B_k\}$ s.t. $Bi=[\omega,\omega']$ is a frequency band start at $\omega$ and end at $\omega'$. We say $X$ Multi-Band Variable-Lag Granger causes $Y$ if $\exists B_i\in\mathcal{B}$, $X^{(B_i)}$ VL-Granger causes $Y^{(B_i)}$ where $(X^{(B_i)},Y^{(B_i)})$ is a pair of signals band-limited at band $B_i$. 
\end{definition}

In the next preposition, we show that the variance of residuals of MB-VLGC has the variance of residuals of traditional VLGC as its upper bound and it can be lower than VLGC's for some cases. 

\begin{proposition}
Let $r^*_{(i)}$ be the residual from fitting $Y^{(i)} \sim VL-G(X^{(i)};P^{(i)})$ and $r^*$ be the residual from fitting a single VL-Granger model to $(X,Y)$, with alignment $P$. Then, $Var(r^*)\geq Var(r^*_{(1)}) + Var(r^*_{(2)})$
\end{proposition}
\begin{proof}

\indent We consider two time series $X$ and $Y$, where: $(X^{(1)},Y^{(1)}) $ and $(X^{(2)},Y^{(2)})$ are pairs of signals band-limited to disjoint frequency ranges $\Omega_1$ and $\Omega_2$, with $\Omega_1\cap \Omega_2= \emptyset$ and each pair admit a variable-lag Granger representation via alignment path $P^{(1)},P^{(2)}$ minimizing residuals. Since $r^*_{(i)}$ be the residual from fitting $Y^{(i)} \sim VL-G(X^{(i)};P^{(i)})$ and $r^*$ be the residual from fitting a single VL-Granger model to $(X,Y)$, with alignment $P$. So, we can show $$Y_t = f_1(X_{t-\delta^{(1)}(t)})+f_2(X_{t-\delta^{(2)}(t)})+\varepsilon_t$$ where $\delta^{(i)}(t)$ is a nonlinear time lag, $f_i$ is the optimal regressor under alignment $\delta^{(i)}$, and $\varepsilon_t$ is i.i.d. noise and let $\hat{Y}_t^{(i)}=f_i(X_{t-\delta^{(i)}(t)})$ We fit a path $\delta(t)$ to model: $\hat{Y}_t = g(X_{t-\delta(t)})$. Then, $X_{t-\delta (t)}= X_{t-\delta (t)}^{(1)}+X_{t-\delta (t)}^{(2)}$. So, $X_{t-\delta (t)}^{(1)}+X_{t-\delta (t)}^{(2)} \geq  X_{t-\delta^{(1)} (t)}^{(1)}+X_{t-\delta^{(2)} (t)}^{(2)}$ this means $g$ cannot access both true lags. Hence, $r^* = Y_t -g(X_{t-\delta(t)})=r^*_{(1)}+r^*_{(2)}+\eta_t$. Then, $$Var(r^*)\geq Var(r^*_{(1)}) + Var(r^*_{(2)})$$
Note that: For the not perfect fitting in $(X,Y)$ case, $\eta_t$ keeps positive, since there must be a information missing. It implies $$Var(r^*)> Var(r^*_{(1)}) + Var(r^*_{(2)})$$.
\end{proof}
With the base case proved in Proposition 2.2, Proposition 2.3 generalizes the result to
$n-$band separation, which we state without proof via induction.
\begin{proposition}
$Var(r^*)\geq \sum_{i=1}^n Var(r^*_{(i)}) $
\end{proposition}
In conclusion, we can say that by dropping a stationary and a VAR assumption, the VL-Granger causality test can be applied on the multi-decomposed bands of the time series. In the case of not perfect fitting in $(X,Y)$ ($\eta_t>0$), $Var(r^*)> \sum_{i=1}^n Var(r^*_{(i)}) $, which implies MB-VLGC performs better than VLGC.  

\section{Methods}
\subsection{Frequency-Band VL-Granger Framework Architecture}

Fig~\ref{fig:overviewframework} shows the high-level overview of our framework.  Our key idea lies in the systematic application of VLGC to frequency-decomposed signals, enabling frequency-specific causality detection that had not been done with previous methods constrained by filter invariance. This allows detection of causal relationships with different temporal delays across frequency bands and frequency-specific causality masked in broadband analysis—capabilities essential for understanding multi-timescale biological and physical processes.

\subsubsection{Framework Pipeline}
There are three steps in the pipeline and the details of each step are below.

\textbf{Stage 1: Frequency Banding.} Input time series are decomposed using bandpass filters into frequency-specific components for each target band. This stage needs to preserves temporal relationships while achieving clean frequency separation essential for reliable causality detection.

\textbf{Stage 2: Causal Inference.} VLGC analysis is applied independently to each frequency band, leveraging the theoretical capability of VLGC to operate on filtered signals without losing causal information. Each band yields detection statistics, lag estimates, and confidence measures.

\textbf{Stage 3: Result Integration.} Band-specific results statistic are systematically combined using established method  and cross-band consistency assessment, producing both overall causality decisions and frequency-specific insights.

\subsubsection{Band decomposition}

As temporal preservation is needed for Granger Causality because it relies on using past observation to predict future values. Zero-phase filtering was implemented using the \texttt{filtfilt} function, which performs forward and backward filtering to eliminate phase distortion \cite{gustafsson1996determining} and preserve its temporal properties.

We employed 4th-order Butterworth bandpass filters due to their maximally flat frequency response in the passband, ensuring minimal amplitude distortion of signal components \cite{oppenheim1999discrete}.  Studies examining neural connectivity consistently employ Butterworth filters for frequency band isolation due to their minimal artifacts and their maximally flat frequency response in the passband \cite{bastos2015tutorial, cohen2014analyzing}.

\subsection{Per-Band Causality Analysis}
\label{subsec:vl_granger}

Our frequency-band method builds upon VLGC \cite{10.1145/3441452}, which extends traditional Granger causality to handle time-varying delays between cause and effect. While traditional Granger causality assumes fixed temporal lags, VLGC allows causal relationships to exhibit variable delays that change dynamically over time, making it more suitable for real-world phenomena where causal timing is not constant.

\subsubsection{Dynamic Time Warping Alignment}

VLGC addresses variable delays by using Dynamic Time Warping (DTW) to find optimal temporal alignments between the cause time series $X$ and effect time series $Y$. DTW identifies the warping path $P^* = (\Delta_1, \Delta_2, ..., \Delta_T)$ that minimizes the cumulative distance between aligned elements:
$$P^* = \arg\min_P \sum_{t} D(X(t-\Delta_t), Y(t))$$
where $\Delta_t$ represents the time delay at step $t$, and $D(\cdot,\cdot)$ is the distance function between aligned points.

\subsubsection{Variable-Lag Regression Framework}

Using the optimal warping path, VLGC constructs time-aligned predictors for regression analysis. The variable-lag causality test compares three nested models:

\textbf{Null Model ($H_0$):} $Y(t) = \sum_{i=1}^{\delta_{max}} a_i Y(t-i) + \epsilon_Y(t)$
\\ \\
\textbf{Fixed-Lag Model ($H_1$):} $Y(t) = \sum_{i=1}^{\delta_{max}} a_i Y(t-i) + \sum_{i=1}^{\delta_{max}} b_i X(t-i) + \epsilon_{YX}(t)$
\\ \\
\textbf{Variable-Lag Model ($H_2$):} $Y(t) = \sum_{i=1}^{\delta_{max}} a_i Y(t-i) + \sum_{i=1}^{\delta_{max}} c_i X^*(t-i) + \epsilon_{VL}(t)$
\\ \\
where $X^*(t-i) = X(t-i+1-\Delta_{t-i+1})$ represents the DTW-aligned version of $X$, and $\delta_{max}$ is the maximum lag considered.

\subsubsection{Hybrid Lag Selection Strategy}

VLGC employs a hybrid approach combining cross-correlation and DTW for optimal lag selection at each time point. The process begins with cross-correlation analysis to identify the globally optimal delay:

$$opt_{delay} = \arg\max_{\tau} |CCF(\tau)| $$ $$= \arg\max_{\tau} \left|\frac{\sum_{t} (X(t-\tau) - \bar{X})(Y(t) - \bar{Y})}{\sqrt{\sum_{t}(X(t) - \bar{X})^2 \sum_{t}(Y(t) - \bar{Y})^2}}\right|$$

where $CCF(\tau)$ is the cross-correlation function at lag $\tau$, and $\bar{X}$, $\bar{Y}$ are sample means.

At each time point $t$, the method selects the lag that minimizes prediction error between candidate alignments:
$chosen_{lag}(t) = \arg\min_{\tau \in \{\tau_{CC}, \tau_{DTW}(t)\}} |Y(t) - X(t - \tau)|$

where $\tau_{CC}$ is the global cross-correlation lag and $\tau_{DTW}(t)$ is the DTW-suggested lag at time $t$. This hybrid strategy combines the global optimality of cross-correlation with the local adaptivity of DTW, yielding time-specific alignments that maximize the causal signal strength while maintaining temporal coherence.

\subsubsection{Statistical Testing Framework}

VLGC employs dual criteria for causality detection:



The method concludes $X$ causes $Y$ if either criterion is satisfied: $p_{F-test} \leq \alpha$ OR $\gamma \geq \gamma_{threshold}$, providing robustness against different signal characteristics and noise conditions.

\subsection{Result integration}

\subsubsection{P-value Combination.} Valid p-values from individual frequency bands are combined using established meta-analysis methods. We employ Fisher's combined probability test as the primary approach:
$$\chi^2 = -2\sum_{i=1}^{k} \ln(p_i)$$
where $k$ is the number of valid frequency bands. Alternative combination methods (Stouffer's method, Bonferroni correction) are also available.

\section{Experimental Evaluation}
\label{sec:experimental}

\subsection{Experimental Setup}
\label{subsec:experimental_setup}

We conducted comprehensive experiments to evaluate our multi-band VL-Granger causality method across two main objectives. First, we tested the method's capability to handle various types of causal relationships. Second, we assessed its performance in identifying causal relationships in real-world datasets. \footnote{All experiments ran on MacBook Air M2,2022 Memory 16 GB SSD 256 GB}

We compared our proposed method against five established causality detection approaches:
\begin{itemize}
    \item \textbf{Variable-lag Granger Causality (VL-GC)} \cite{10.1145/3441452}: The foundational variable-lag method that our approach extends
    \item \textbf{Granger Causality (GC)} \cite{granger1969investigating}: Traditional fixed-lag Granger causality as implemented in standard econometric packages
    \item \textbf{Variable-lag Transfer Entropy (VL-TE)} \cite{10.1145/3441452}: Non-linear extension of transfer entropy with variable-lag capability
    \item \textbf{Transfer Entropy (TE)} \cite{BEHRENDT2019100265}: Standard transfer entropy with fixed-lag assumptions
    \item \textbf{PCMCI+} \cite{runge2020discovering}: State-of-the-art conditional independence-based causality detection method
    \item \textbf{Granger-Geweke (GG)} \cite{farne2022bootstrap}: Spectral decomposition of Granger causality that measures causal strength in the frequency domain using R-CRAN Package \textit{grangers}.
\end{itemize}

For all experiments, we set the significance level $\alpha = 0.01$ and the BIC difference ratio threshold $\gamma = 0.6$ by default, unless explicitly stated otherwise.

\subsection{Datasets}

\subsubsection{Synthetic Datasets}
\label{subsubsec:synthetic_datasets}

We generated a comprehensive synthetic dataset consisting of 240 time series files to systematically evaluate causality detection performance across different scenarios. The dataset is divided into two main categories: datasets with ground truth causality (120 files) and datasets without causality (120 files) for testing false positive rates.

\paragraph{False Positive Test Datasets (120 files):}
Independent time series generated from standard normal distributions with no causal relationships (\textit{Random Noise}).

\paragraph{True Positive Test Datasets (120 files):}
The causality datasets are further divided into four distinct types:

\textbf{Basic Causation (30 files):} Simple linear causality where the effect time series is generated from the cause time series with a fixed lag of 20 samples and random coupling coefficient.

\textbf{Variable-lag Causation (30 files):} The effect follows the cause with discrete lag periods that change over time. The lag switches between three values (12, 16, and 20 samples) with 2-4 regime changes per time series.

\textbf{Broadband Causation (30 files):} Both cause and effect contain rich frequency content spanning 1-50 Hz, with the effect generated by lagging the broadband cause signal by 7 samples. This tests performance on frequency-rich signals where traditional methods may struggle.

\textbf{Multi-frequency Causation (30 files):} Complex scenarios where different frequency components (10, 40, and 80 Hz) exhibit different causal lags (15, 8, and 4 samples respectively). This directly tests the method's ability to detect frequency-specific causal relationships.

All synthetic time series were generated with lengths between 500-2000 samples at a sampling rate of 250 Hz. Random seeds were set for reproducibility, and coupling coefficients were drawn from normal distributions to ensure realistic signal-to-noise ratios. The variable-lag datasets use discrete lag periods rather than continuous variation to create detectable patterns that reflect realistic biological and physical processes. The code for data generation is available in the code and datasets link.

\subsubsection{Real-world Datasets}
\label{subsubsec:realworld_datasets}

We evaluated our method on four established real-world datasets from diverse domains to assess its practical applicability and validate frequency-specific causality detection capabilities. The $X$ represents case and $Y$ represents effect in the setting.

\paragraph{Old Faithful Geyser:} This classic dataset contains eruption duration ($X$) and inter-eruption intervals $Y$ from Old Faithful geyser \cite{azzalini1990look}. We tested causality between consecutive eruption characteristics, where previous eruption duration may influence the next inter-eruption interval. The dataset contains 298 observations and serves as a benchmark for temporal causality methods.

\paragraph{Chicken and Egg Prices:} Economic time series data examining the price relationship between egg $X$ and chicken $Y$ markets. This dataset~\cite{zeileis2002diagnostic} tests our method's ability to detect economic causality relationships that may operate at different frequency scales, from short-term market fluctuations to longer-term supply-demand dynamics. The time series length is 54 time steps.

\paragraph{Gas Furnace:} Industrial process data measuring gas consumption rates $X$ and corresponding CO$_2$ output $Y$~\cite{box2015time}. This controlled system provides a clear causal relationship where gas input drives CO$_2$ production, with the dataset containing 296 time steps. The known causal direction makes this ideal for validation.

\paragraph{EEG Motor Imagery:} Neurophysiological data from the EEG Motor Movement/Imagery Dataset \cite{schalk2004bci2000} available through PhysioNet. We analyzed causality between electrodes FC3 and FC5, which are positioned over motor cortex regions. These electrodes are expected to show coordinated activity during motor tasks, with potential frequency-specific interactions in different neural rhythms (alpha, beta, gamma bands). The high sampling rate (250 Hz) and rich frequency content make this dataset particularly suitable for testing frequency-band causality detection.

These datasets span multiple domains (geophysical, economic, industrial, and neurophysiological) and provide diverse signal characteristics including different sampling rates, noise levels, and frequency content. The neurophysiological data is especially relevant for demonstrating the practical value of frequency-specific causality analysis, as different neural rhythms are known to carry distinct functional information.
\section{Results}

We outline the result of testing our approach against five established methods on both synthetic and real-world datasets.

\subsection{Synthetic Dataset Performance}

\begin{table}[htbp]
\centering
\caption{Performance Comparison Across Dataset Types (Accuracy and F1-Scores). The elements in the table are accuracy in every row ecept the last row that the elements are F1-scores.}
\label{tab:performance_comparison}
\scriptsize
\setlength{\tabcolsep}{3pt}
\begin{tabular}{|p{2.0cm}|c|c|c|c|c|c|c|}
\hline
\multicolumn{1}{|c|}{} & \multicolumn{7}{c|}{\textbf{Accuracies of Methods}} \\
\hline
\textbf{Datasets} & \textbf{MB-VL} & \textbf{VLGC} & \textbf{GC} & \textbf{TE} & \textbf{VLTE} & \textbf{PCMCI+} & \textbf{GG} \\
\hline
Following relation & 0.833 & 0.733 & \textbf{1.000} & 0.367 & 0.933 & \textbf{1.000} & 0.400 \\
\hline
Variable-lag & 0.767 & 0.600 & 0.233 & 0.333 & 0.867 & \textbf{1.000} & 0.467 \\
\hline
Broadband lag & 0.867 & 0.433 & 0.833 & 0.467 & 0.933 & \textbf{0.967} & 0.000 \\
\hline
Multifrequency lag & \textbf{0.933} & 0.167 & 0.267 & 0.567 & 0.533 & 0.900 & 0.133 \\
\hline
Random Noise & 0.750 & \textbf{1.000} & \textbf{1.000} & 0.592 & 0.617 & 0.300 & 0.525 \\
\hline
\hline
\textbf{Overall F1-score} & \textbf{0.810} & 0.742 & 0.792 & 0.512 & 0.717 & 0.725 & 0.388 \\
\hline
\end{tabular}
\end{table}

Our method achieved the highest overall F1-score of 0.810. Most notably, our method excelled on multi-frequency datasets with an accuracy of 0.933, significantly outperforming traditional methods that struggled with frequency-specific causal relationships. This demonstrates the core strength of our approach: the ability to detect causality that operates differently across frequency bands. For the case of positive class (X causes Y),  PCMCI+ performed the best and it performed slightly better than our method (MB-VL). However, it performed worse than our method in the random noise case (X does not cause Y).  The VLGC and GC performed the best in the random noise case but they were unable to deal with many types of lag case especially the Multifrequency lag case. The VLTE performed better than TE and our method in most of the cases but it performed dramaically poor compare to our method in the Multifrequency lag case. Lastly, GG had the worst performance. 

\subsection{Impact of Frequency Band Configuration}

A critical aspect of our method is the selection of frequency bands, which directly impacts performance depending on the underlying signal characteristics. Table \ref{tab:band_strategy} illustrates how different band configuration strategies affect performance.

\begin{table}[htbp]
\centering
\caption{Performance Impact of Band Configuration Strategy}
\label{tab:band_strategy}
\footnotesize
\begin{tabular}{|l|c|c|c|}
\hline
\textbf{Configuration} & \textbf{Strategy} & \textbf{Multi-freq} & \textbf{Overall F1} \\
\hline
Single Band & Broadband analysis & 0.167 & 0.742\\
\hline
Two Bands & Optimal balance & 0.933 & 0.810 \\
\hline
EEG Bands & Frequency-specific & 1.000 & 0.618 \\
\hline
\end{tabular}
\end{table}

The two-band configuration (1-80 Hz and 81-120 Hz) provided the optimal balance across all dataset types, achieving the highest overall F1-score while maintaining excellent performance on multi-frequency datasets. This configuration represents a principled frequency division that separates low-frequency oscillatory dynamics from higher-frequency transient processes.

Interestingly, the EEG-specific six-band configuration achieved perfect performance (1.000) on multi-frequency datasets but lower overall performance (0.618) due to over-segmentation of simpler causal relationships. This illustrates an important principle: Optimal band selection depends on the expected characteristics of the causal relationships and available domain knowledge.

For applications where signal characteristics are unknown, our two-band configuration provides robust performance. However, when domain-specific knowledge is available—such as in EEG analysis where neural oscillations operate in well-defined frequency ranges—informed band selection allows users to investigate frequency-specific causal relationships that correspond to their research questions and domain understanding, revealing underlying mechanisms that would remain hidden in typical causal analysis.

\subsection{Frequency-Specific Lag Detection}

Our method provides insights into the temporal dynamics within different frequency bands beyond causality detection. Table \ref{tab:lag_performance} demonstrates our method's ability to accurately detect causal lags across different neural frequency bands from the multi-frequency causation dataset.

\begin{table}[htbp]
\centering
\caption{Lag Detection Performance by Frequency Band}
\label{tab:lag_performance}
\footnotesize
\setlength{\tabcolsep}{5pt}
\begin{tabular}{|l|c|c|c|c|}
\hline
\textbf{Band} & \textbf{True Lag} & \textbf{Inferred Lag} & \textbf{Sig. Lag} & \textbf{Error} \\
\hline
Alpha & 15 & $12.3 \pm 6.0$ & $11.6 \pm 7.0$ & $1.2 \pm 0.4$ \\
\hline
L Gamma & 8 & $4.3 \pm 3.9$ & $5.0 \pm 0.0$ & $3.0 \pm 0.0$ \\
\hline
H Gamma & 4 & $2.4 \pm 3.3$ & $2.4 \pm 3.3$ & $2.5 \pm 1.8$ \\
\hline
\end{tabular}
\end{table}

The results show that our method successfully detects frequency-specific temporal delays with high accuracy. High gamma frequencies (30-100 Hz) demonstrated the most precise lag detection, closely matching the true lag of 4 samples. Alpha band detection showed more variability, which is consistent with the known properties of alpha oscillations in neural systems. This frequency-specific lag information provides valuable insights that would be lost in traditional broadband analysis.

\subsection{Real-world Dataset Validation}
\begin{figure*}[ht!]
\centering
\includegraphics[width=1.5\columnwidth]{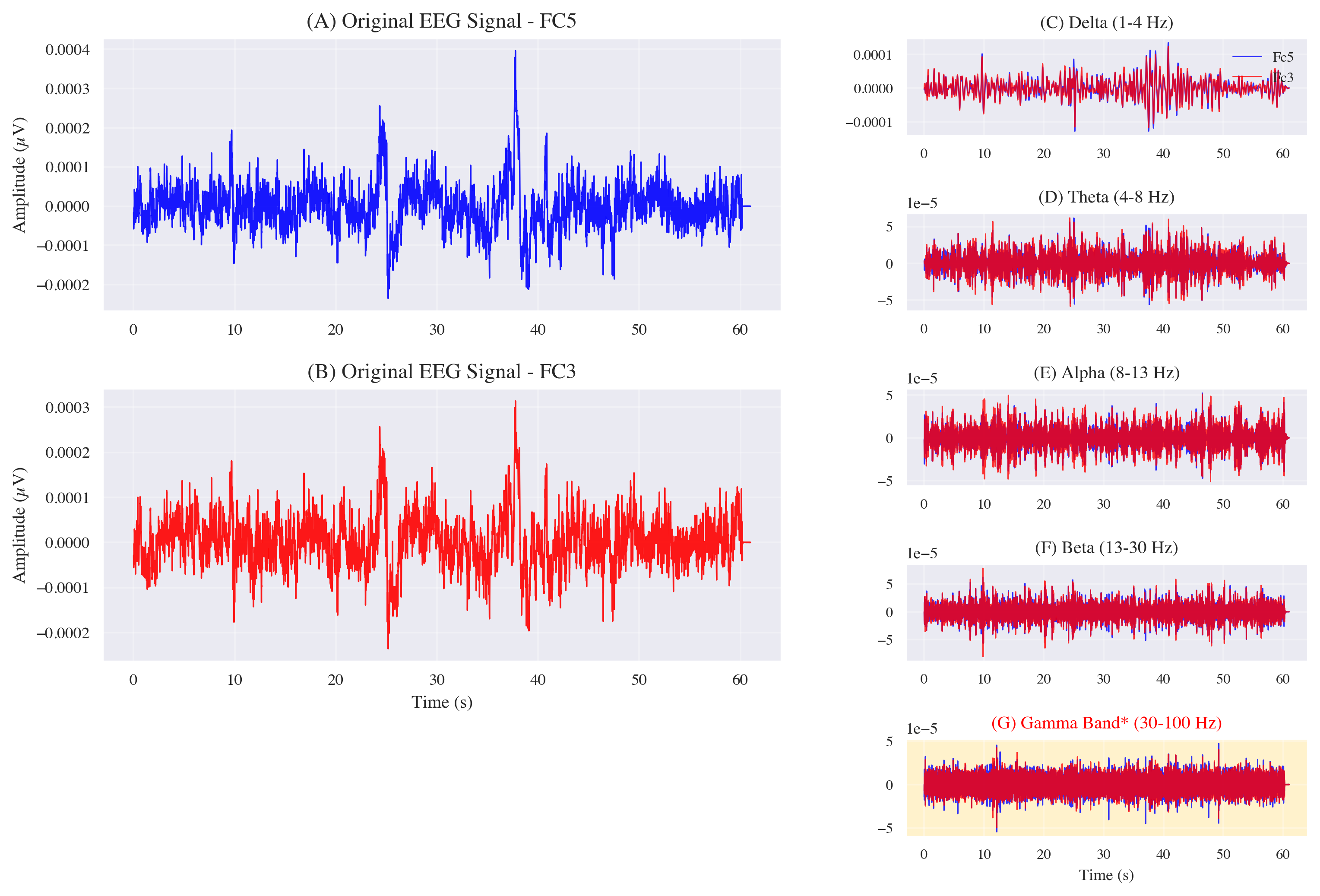} 
\caption{The time series from EEG motor imagery dataset: the red is FC3 and the blue is FC5. (Left) the original time series. (right) the band-limited time series that were separated by frequency bands. Only the gamma band (the right bottom) has the bidirectional VL-Granger causation which is consistent with the ground truth. }
\label{fig:eegRes}
\end{figure*}

To validate practical applicability, we tested our method on four established real-world datasets from diverse domains. Table \ref{tab:realworld_results_minimal} presents the causality detection results, where 1 indicates successful causality detection and 0 indicates no detected causality.

\begin{table}[htbp]
\centering
\caption{Causality Detection Results on Real-world Datasets}
\label{tab:realworld_results_minimal}
\footnotesize
\setlength{\tabcolsep}{3.5pt}
\begin{tabular}{|l|c|c|c|c|c|c|c|}
\hline
& \multicolumn{7}{c|}{\textbf{Methods}} \\
\hline
\textbf{Case} & \textbf{MB-VL} & \textbf{VLGC} & \textbf{G} & \textbf{TE} & \textbf{VLTE} & \textbf{PCMCI} & \textbf{GG}  \\
\hline
EEG & 1 & 0 & 0 & 1 & 1& 1 & 0\\
\hline
Chick. Egg & 1 & 1 & 1 & 1 & 1 & 1 & 1\\
\hline
Old ffg. & 1 & 1 & 0 & 0 & 1 & 1 & 0 \\
\hline
Gas fur. & 1 & 1 & 1 & 1 & 1 & 1 & 1\\
\hline
\end{tabular}
\end{table}

Our method achieved perfect causality detection across all four real-world datasets, demonstrating robust performance in practical applications. The EEG motor imagery dataset particularly showcases the value of frequency-specific analysis, where our method successfully detected causality between FC3 and FC5 electrodes while the traditional Granger causality failed. This result aligns with neuro-scientific understanding that motor cortex regions exhibit coordinated activity across different neural frequency bands.

In Fig.~\ref{fig:eegRes} (the right bottom), only the gamma band has the bidirectional VL-Granger causation which is consistent with the ground truth while there were no causal relation in other bands. This implies that the gamma band might be the main contributor for the causal relations between FC3 and FC5 electrodes.

The diverse nature of these datasets—spanning geophysical (Old Faithful geyser), economic (chicken and egg prices), industrial (gas furnace), and neuro-physiological (EEG) domains—demonstrates the broad applicability of our frequency-band approach. Notably, PCMCI+ failed on the economic dataset, highlighting scenarios where conditional independence assumptions may be violated but frequency-specific causality relationships still exist. 


\section{Conclusion}
We formalized Multi-Band Variable-Lag Granger Causality (MB-VLGC) and proposed a novel framework that generalizes traditional variable-lag Granger causality (VLGC) by explicitly modeling frequency-dependent causal delays.  We provided a formal definition of MB-VLGC,  demonstrated its theoretical soundness, and proposed an efficient inference pipeline. 

According to the results, our Multi-Band VL-Granger causality method addressed fundamental limitations in existing causality detection approaches by enabling frequency-specific analysis. The method achieved the better overall performance (F1 = 0.810) than others (Table \ref{tab:performance_comparison}) while providing the flexibility to adapt to domain-specific requirements through informed band selection. The key finding is that optimal performance requires matching the frequency band configuration to the expected characteristics of causal relationships, with our two-band default configuration providing robust performance when domain knowledge is limited. Code and datasets are publicly available in Code and datasets section.

\begin{links}
\link{Code and datasets}{https://anonymous.4open.science/r/mbvlgranger-ED89/README.md}
\end{links} 

\bibliographystyle{apalike}

\end{document}